\documentclass[runningheads]{llncs}
\usepackage[T1]{fontenc}

\let\labelindent\relax

\makeatletter
\let\NAT@parse\undefined
\makeatother

\usepackage{graphicx}
\usepackage{multicol}
\usepackage{xcolor}
\usepackage{thmtools}
\usepackage{amsmath}
\usepackage{amsthm}
\usepackage{hyperref}
\definecolor{linkcolor}{rgb}{0.45,0.05,0.05}
\definecolor{citecolor}{rgb}{0.05,0.45,0.45}
\definecolor{urlcolor}{rgb}{0.05,0.05,0.45}
\hypersetup{
  linkcolor  = linkcolor,
  citecolor  = citecolor,
  urlcolor   = urlcolor,
  colorlinks = true,
}

\usepackage{environ}

\newif\ifmoveprooftoend
\moveprooftoendtrue
\newcommand\showDeferredProofs{}
\makeatletter
\NewEnviron{movableProof}[1][Proof]{%
        \ifmoveprooftoend
                \edef\next{\noexpand\g@addto@macro\noexpand\showDeferredProofs{%
                        \noexpand\begin{proof}[#1]\unexpanded\expandafter{\BODY}\noexpand\end{proof}}}
                \next
        \else
                \begin{proof}[#1]\BODY\end{proof}
        \fi}{}
\makeatother

\newcommand{\gobble}[1]{}
\newcommand{\gobblexor}[2]{#2} 

\usepackage[font=small,labelfont=bf]{caption}
\usepackage{subcaption}
\usepackage{mathtools}
\usepackage{algorithm}
\usepackage[noend]{algorithmic}
\usepackage{listings}
\usepackage{xcolor}
\usepackage{comment}
\usepackage[utf8]{inputenc}
\usepackage{xspace}
\usepackage{mathrsfs}
\usepackage{tabularx}
\usepackage{xcolor}
\usepackage{graphics}
\usepackage{enumitem}
\usepackage{amsfonts}
\usepackage{amssymb}
\usepackage{stmaryrd}

\theoremstyle{definition} 
\newtheorem{x}{X}
 
\newtheorem{z}{Z} 
\newtheorem{q}{Q} 
\newtheorem{rr}{R} 
 
\newtheorem{definition}[x]{Definition} 
\newtheorem{corollary}[x]{Corollary} 
\newtheorem{problem}[x]{Problem} 
\newtheorem{question}[q]{Question} 
 
\newtheorem{lemma}[x]{Lemma} 
\newtheorem{theorem}[x]{Theorem} 
\newtheorem{example}[x]{Example} 
\newtheorem{property}[x]{Property}

\newtheorem{construction}[x]{Construction}
\newtheorem{proposition}[x]{Proposition}

\newtheorem{conjecture}[z]{Conjecture}
\newtheorem{remark}[rr]{Remark}

\definecolor{grey}{rgb}{0.5,0.5,0.5}
\definecolor{darkgrey}{rgb}{0.15,0.15,0.15}
\definecolor{darkblue}{rgb}{0.05,0.05,0.5}
\definecolor{darkgreen}{rgb}{0.05,0.5,0.05}
\definecolor{darkestgreen}{rgb}{0.5,0.0,0.5}
\definecolor{darkorange}{rgb}{0.5,0.25,0.00}

\providecommand{\defemp}[1]{\emph{#1}} 
\newcounter{tecounter}
\providecommand{\zc}[3]{\ensuremath{({#1}; {#2})}}

\providecommand{\extension}[2]{\ensuremath{\mathcal{L}_{#1}(#2)}}

\newcommand*{\probleminternal}[4]{
    {\small
	\par
	\medskip
	\noindent\fbox{\parbox{0.98\columnwidth}{
		\hspace*{-1.9em}\textbf{#4:} {#1} \\[0.05in]
		\renewcommand{\tabcolsep}{2pt}
		\begin{tabularx}{0.97\linewidth}{rX}
			~\emph{Input:} & #2 \\
			~\emph{Output:} & #3
		\end{tabularx}
	}}}
	\par
	\medskip
	\par
}
\newcommand*{\probleminternalwithnote}[5]{
    {\small
	\par
	\medskip
	\noindent\fbox{\parbox{0.98\columnwidth}{
		\hspace*{-1.9em}\textbf{#5:} {#1} \\[0.05in]
		\renewcommand{\tabcolsep}{2pt}
		\begin{tabularx}{0.97\linewidth}{rX}
			~\emph{Input:} & #2 \\
			~\emph{Output:} & #3 \\
   \multicolumn{2}{l}{#4}
		\end{tabularx}
        }}}
	\par
	\medskip
	\par
}

\newcommand*{\ourproblem}[3]{\probleminternal{#1}{#2}{#3}{\quad~ Problem}}
\newcommand*{\ourproblemwithnote}[4]{\probleminternalwithnote{#1}{#2}{#3}{#4}{\quad~ Problem}}

\providecommand{\compatible}{\ensuremath{\sim_{c}}}
\providecommand{\up}{\ensuremath{\mathrlap{\raisebox{1pt}{\ensuremath{{}^{\,\phantom{\ast}}}}}{\,\prec\,}}}
\providecommand{\strictupstream}{\ensuremath{\mathrlap{\raisebox{1pt}{\ensuremath{{}^{\,\ast}}}}{\,\prec\,}}}

\newcommand{\scalingsize}{0.85}
\newcommand{\fmp}{\scalebox{\scalingsize}{{\sc FM}}\xspace}
\newcommand{\fm}[1]{\scalebox{\scalingsize}{{\sc FM}}$({#1})$\xspace}
\newcommand{\mzccp}{\scalebox{\scalingsize}{\sc MZCC}\xspace}
\newcommand{\mzcc}[2]{\scalebox{\scalingsize}{\sc MZCC}$({#1},{#2})$\xspace}

\providecommand{\nptext}{{\footnotesize\textsf{NP}}}

\providecommand{\nphard}{{\nptext -hard}\xspace}

\providecommand{\coll}[1]{\ensuremath{\mathbb{#1}}}
\providecommand{\struct}[1]{\ensuremath{\mathscr{#1}}\xspace}
\providecommand{\cover}[1]{\coll{#1}}

\providecommand{\neighbor}[2]{\ensuremath{\mathcal{N}_{#1}(#2)}}

\newcommand\superrestr[3]{{
  \left.\kern-\nulldelimiterspace 
  #1 
  \vphantom{\big|} 
  \right|_{#2}^{#3} 
  }}

\newcommand\restr[2]{{
  \left.\kern-\nulldelimiterspace 
  #1 
  \vphantom{\big|} 
  \right|_{#2} 
  }}

\newcommand{\citep}[1]{\cite{#1}}
\newcommand{\citet}[1]{\cite{#1}}

\renewcommand{\emptyset}{\varnothing}

\makeatletter
\newcommand*\bigcdot{\mathpalette\bigcdot@{.7}}
\newcommand*\bigcdot@[2]{\mathbin{\vcenter{\hbox{\scalebox{#2}{$\m@th#1\bullet$}}}}}
\makeatother

\providecommand{\nptext}{{\footnotesize\textsf{NP}}}

\providecommand{\nphard}{{\nptext -hard}\xspace}

\providecommand{\cn}{comparable neighborhoods\xspace}

\newcommand\shortvspace[1]{\ifdefined\arxiv\else \vspace*{#1}\fi}

\newcommand{\fpt}{FPT\xspace}

\newcommand\pr[2]{\ensuremath{\{{#1}, {#2}\}}}

\newcommand{\zipcoll}[1]{\textcolor{teal!50!black}{\ensuremath{\mathcal{#1}}}}
\newcommand{\tgt}{\scalebox{0.5}{\text{tgt}}}
\newcommand{\src}{\scalebox{0.5}{\text{src}}}
\newcommand{\Zpairs}{\ensuremath{{Z^2}}}
\newcommand{\ZpairsT}{\ensuremath{Z_{\tgt}^2}}
\newcommand{\ZpairsS}{\ensuremath{Z_{\src}^2}}

\newcommand{\prior}{{\scalebox{0.7}{\text{prior}}}}
\newcommand{\cpy}{{\scalebox{0.7}{\text{copy}}}}
\newcommand{\extra}{{\scalebox{0.7}{\text{new}}}}

\providecommand{\expand}[1]{{\operatorname{expand}}({#1})}

\providecommand{\cri}[1]{{\footnotesize\textbf{#1.}}}

\providecommand{\newvertexSet}[1]{\ensuremath{\mathsf{v}_{\extra}^{\scalebox{0.67}{${#1}$}}}}
\providecommand{\newvertexS}[1]{\ensuremath{\mathsf{v}_{\extra}^{\scalebox{0.67}{$\!\{{#1}\}\!$}}}}
\providecommand{\newvertex}[2]{\ensuremath{\mathsf{v}_{\extra}^{\scalebox{0.67}{$\!\{{#1},{#2}\}\!$}}}}
\providecommand{\minion}{\scalebox{0.5}{\text{ON}}}
\providecommand{\minioff}{\scalebox{0.5}{\text{OFF}}}
\providecommand{\Selon}[2]{\coll{#1}^{\minion}_{#2}}

\providecommand{\Seloff}[2]{\coll{#1}^{\minioff}_{#2}}

\providecommand{\hyp}{\circ}

\definecolor{domaincolor}{rgb}{0.0,0.0,0.4}

\newcommand{\D}{\ensuremath{{\color{domaincolor}D}}\xspace}

\newcommand{\R}{\ensuremath{{\color{domaincolor}R}}\xspace}
\newcommand{\uR}{\ensuremath{{\color{domaincolor}\underline{R}}}\xspace}

\makeatletter
\newcommand*{\Dbar}{}%
\DeclareRobustCommand*{\Dbar}{%
  \mathpalette\@Dbar{}%
}
\newcommand*{\@Dbar}[2]{%
  \sbox0{$#1J\m@th$}%
  \sbox2{$#1\D\m@th$}%
  \rlap{%
    \hbox to\wd2{%
      \hfill
      \textcolor{domaincolor}{$\overline{%
        \vrule width 0pt height\ht0 %
        \kern\wd0 %
      }$}%
    }%
  }%
  \copy2 %
}
\makeatother

\newcommand{\eD}{\Dbar}
\newcommand{\eSelon}[2]{\Selon{#1}{#2}}

\newcommand{\on}{\textsc{on}\xspace}
\newcommand{\off}{\textsc{off}\xspace}

\newcommand{\onpairs}{\on pairs\xspace}

\newcommand{\offpairs}{\off pairs\xspace}

\newcounter{sublOne}

\newcounter{sublTwo}

\newcommand{\srcSet}[1]{\ensuremath{\mathscr{S}_{#1}}}
\newcommand{\destSet}[1]{\ensuremath{\mathscr{D}_{#1}}}
\newcommand{\interiorZ}{\ensuremath{\mathcal{I}}}
\newcommand{\iterIdx}[1]{\ensuremath{^{({#1})}}}

\newcommand{\upSet}[2]{\ensuremath{\mathcal{U}_{#1}}({#2})}
\newcommand{\downSet}[2]{\ensuremath{\mathcal{D}_{#1}}({#2})}
\newcommand{\prG}[1]{\ensuremath{\mathbb{#1}}}
\newcommand{\fxn}[1]{\ensuremath{\textcolor{purple!50!black}{\textrm{\textsc{#1}}}}}
\newcommand{\enum}{\fxn{EnumDS}}

\newcommand{\cyceq}{\ensuremath{\equiv}}

\def\BibTeX{{\rm B\kern-.05em{\sc i\kern-.025em b}\kern-.08em
    T\kern-.1667em\lower.7ex\hbox{E}\kern-.125emX}}

\title{A fixed-parameter tractable algorithm for combinatorial filter reduction}

\author{Yulin Zhang\inst{1} \and Dylan A. Shell\inst{2}}

\institute{
Amazon Robotics, North Reading, MA, USA.
\email{\tt\small zhangyl@amazon.com}\and
Texas A\&M University, College Station, TX, USA.
\email{\tt\small dshell@tamu.edu}
}

\date{}

\begin{document}

\maketitle


\begin{abstract}
What is the minimal information that a robot must retain to achieve its task?
To design economical robots, the literature dealing with reduction of combinatorial filters approaches this problem algorithmically.
As lossless state compression is \nphard, prior work has examined, along with minimization algorithms, a variety of special cases in which specific properties enable efficient solution.
Complementing those findings, this paper refines the present understanding from the perspective of parameterized complexity. 
We give a fixed-parameter tractable algorithm for the general reduction problem by exploiting a transformation into clique covering. 
The transformation introduces new constraints that arise from sequential
dependencies encoded within the input filter\,---\,some of these constraints can be repaired, 
others are treated through enumeration.
Through this approach, we identify parameters affecting filter reduction
that are based upon inter-constraint couplings (expressed as a notion of their height and width), 
which add to the structural parameters present in the unconstrained
problem of minimal clique covering.
Compared with existing work, we 
precisely identify and quantitatively characterize those features that contribute to the problem's hardness:
given a problem instance, the combinatorial core may be a
fraction of the instance's full size, with a small subset of
constraints needing to be considered, and even those may have directly 
identifiable couplings that collapse degrees of freedom in the
enumeration.

\end{abstract}

\newcommand\blfootnote[1]{%
  \begingroup
  \renewcommand\thefootnote{}\footnote{#1}%
  \addtocounter{footnote}{-1}%
  \endgroup
}

\blfootnote{This work was done prior to Y. Zhang joining Amazon Robotics.}

\vspace*{-28pt}
\section{Introduction}
\vspace*{-4pt}

The design of robots that are simple is important not only because small resource footprints often translate into money saved, but also because parsimony can be enlightening.
In fact, the pursuit of minimalism has a long history in robotics (cf.~\cite{connell1990minimalist,goldberg96minimalism,mason93kicking}).
But the elegance in that 
prior work was obtained, mostly, through human nous rather than computational tools---our work pursues the latter avenue.

Combinatorial filters are a general and abstract model of stateful devices that take a
stream of sensor readings as input and, processing sequentially, produce a
stream of outputs.  They have found direct application in describing estimators (e.g.,
tracking agents in an environment\,\cite{tovar2014combinatorial}) and
also as representations for sensor-based plans (e.g., for navigation, or manipulation for part
orientating\,\cite{okane17concise}).
Figure~\ref{fig:examples} provides specific concrete examples from the literature showing
different scenarios and the associated filters.
In the context of the present paper, what is interesting about
combinatorial filters is that (unlike, say, Bayesian estimators)
they are objects which themselves can be modified by algorithms.
In our view,
the fundamental information processing task faced by a robot can oftentimes be abstractly
represented via a combinatorial filter, so specific obstructions to
tractability have significance beyond mere applications; for instance, they
speak to the challenge of niche fit as optimization
under resource constraints.

\newcommand{\exfig}{
\begin{figure}[t]
\vspace*{3pt}
\begin{subfigure}[b]{0.32\textwidth}
\centering
\includegraphics[width=0.7\linewidth]{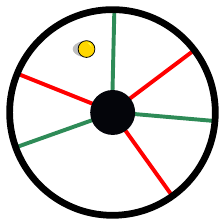}

(a)

\end{subfigure}
\hfill
\begin{subfigure}[b]{0.32\textwidth}
\centering
\includegraphics[width=0.9\linewidth]{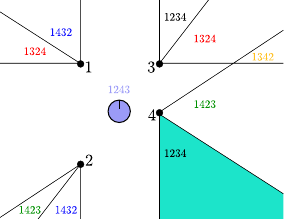}

(c)

\end{subfigure}
\hfill
\begin{subfigure}[b]{0.32\textwidth}
\centering
\includegraphics[width=0.85\linewidth]{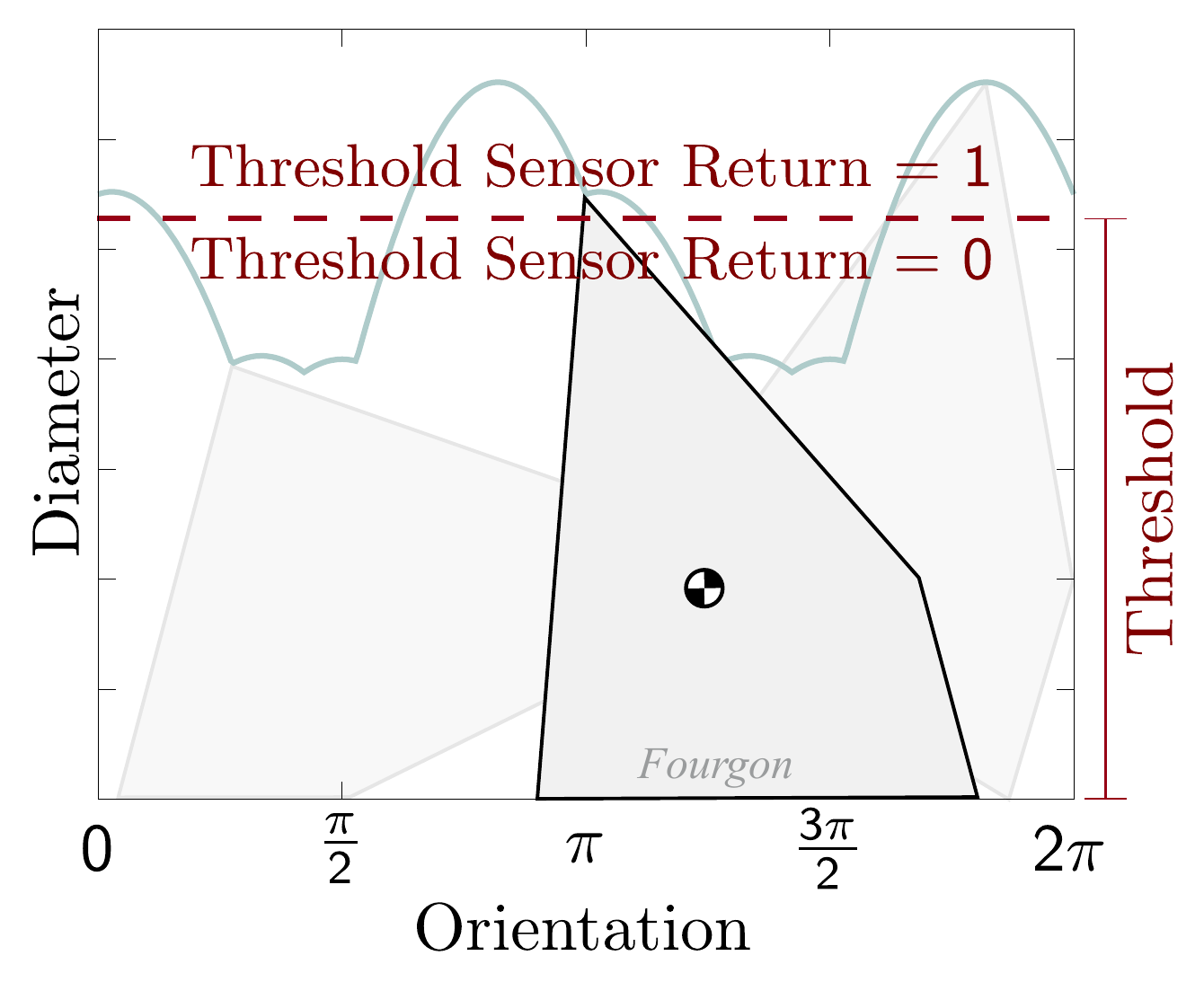}

(e)

\end{subfigure}
\phantom{space}\\[-4pt]
\begin{subfigure}[c]{0.32\textwidth}
\centering
\includegraphics[width=0.8\linewidth]{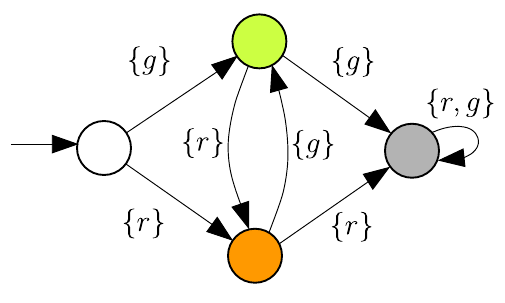}

(b)
\end{subfigure}
\hfill
\begin{subfigure}[c]{0.32\textwidth}
\centering
\includegraphics[width=\linewidth]{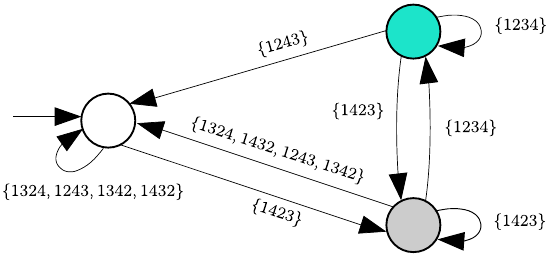}

(d)
\end{subfigure}
\hfill
\begin{subfigure}[c]{0.32\textwidth}
\centering
\includegraphics[width=\linewidth]{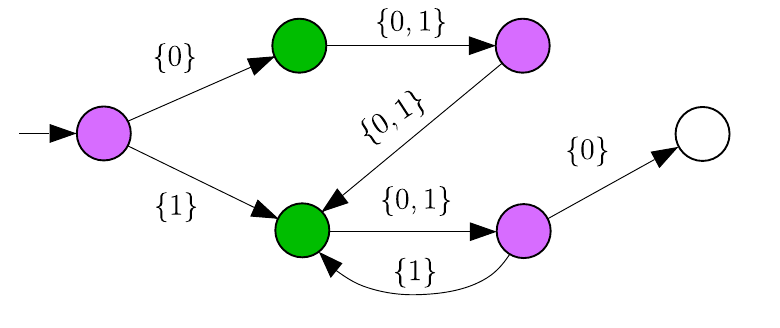}
(f)
\end{subfigure}
\vspace*{-6pt}
\caption{Diverse examples of combinatorial filters.
Sakcak et~al.\ \cite{sakcak2024mathematical} consider a circular environment like that shown in (a) with
two types of break-beam sensors that trigger when crossed by an agent.
They derive the 4-state filter, depicted in (b), that outputs grey when it detects that strictly clockwise/anticlockwise motion has been violated.
Scenario (c) is re-drawn from \cite{rahmani21equivalence}, where the robot observes only the cyclic ordering
of 4 landmarks. Their 3-state filter, depicted in (d), 
determines definitively whether the robot's current location is within the cyan region. 
In (e), the task of orienting a polygon with a squeeze-gripper, based upon~\cite{Taylor1988SensorbasedMP},
is expressed as feedback plan (f),  
with green encoding the action of rotating the gripper by $65^\circ$, and purple the squeeze action (reproduced from~\cite{zhang22sso}). \label{fig:examples}}
\vspace*{-20pt}
\end{figure}
}

\exfig

\vspace*{-12pt}
\subsection{Related Work}
\vspace*{-4pt}
The idea of determining fundamental limits, such as necessary
information and performance bounds, is a topic 
receiving renewed attention (for instance, see \cite{majumdar2022,sakcak2024mathematical}), 
especially when such analysis can be conducted via automated means.
A key problem is that of taking a combinatorial filter and 
compressing it to form an equivalent filter, but with the fewest
states. 
Regrettably,
this is \nphard~\cite{okane17concise}.
A direction of work has sought to identify
special cases~\cite{saberifar17special,zhang23general},
to employ ILP and SAT techniques~\cite{rahmani2020integer,zhang2021accelerating},
and to focus on special types of reductions which may give inexact
solutions~\cite{rahmani21equivalence,saberifar18improper}.
(The state-of-the-art practical method for exact combinatorial filter minimization, at the time of writing, appears 
in~\cite{zhang2021accelerating}.)
The very closely related work of \cite{saberifar17special} identifies some parameters for which
filter reduction is not \fpt (namely: output size, treewidth, maximum number of appearances of any observation); they also give an \fpt algorithm for a stage-by-stage reduction problem (not identical to the problem we study, see~\cite[Lemma~10, pg.~95]{zhang20cover})
which has parameters distinct from those we examine. 

\subsection{Contributions}

The present paper gives a new fixed-parameter tractable (\fpt) algorithm
for the filter reduction problem. 
Such algorithms are so named because 
they involve the identification of specific parameters as dimensions that
characterize problem instances as inputs to the algorithm.  These algorithms
have attractive (polynomial) performance when the instances
are scaled up but with the parameters held constant.
As a formal tool, they provide more fine-grained treatment than merely
showing the problem is \nphard~\cite{niedermeier2006invitation}.
The primary significance of the algorithm we introduce is that it 
highlights specific structural aspects of instances that make 
minimization difficult. 
Put another way: when the parameters identified are
bounded, that which remains is a characterization of easy (i.e. polynomial
time) filter minimization instances.
For instance, easy filters to minimize are those where the number of 
zipper constraints (see Definition~\ref{def:zipper_constraint}) is small. 
Or, when the set of zipper constraints
might be large, minimization can be easy if
the vast majority satisfy the repairability property 
we describe (see Definition~\ref{def:nc}),
while being disconnected from those which do not. 
Also, 
when non-repairable constraints, owing to constraint inter-dependencies,
form long chains, or---even better---cycles, the problem is simplified. Each of these help reduce the (worst-case) cost of solution.

The perspective we emphasize, thus, is that the
algorithm provides a new complexity-theoretic insight by teasing apart specific structural factors affecting the hardness of filter reduction; 
though we do not 
put the
algorithm to
practical use, there is no basis to presume that it
is impractical, either. 


\section{Preliminaries}

\begin{definition}[filter~\cite{setlabelrss}]
A \defemp{deterministic filter}, or just \defemp{filter},
is a 6-tuple $\struct{F} = (V, v_0, Y, \tau, C, c)$, with $V$ a non-empty
finite set of states, $v_0$ an initial state, $Y$ the set of
observations, $\tau: V\times Y \hookrightarrow V$ 
the partial function describing transitions, $C$ the set of outputs, and $c: V\to
C$ being the output function.  
\end{definition}

One \emph{traces} a finite observation sequence $y_0 y_1 \dots y_k$ 
on a filter by starting at state $v_0$, and repeatedly following the edge labeled by 
$y_i$ to arrive at $v_{i+1} = \tau(v_i, y_i)$. The filter's \emph{output} is 
$c(v_k) \in C$ obtained from the last state visited.

\smallskip

\ourproblem{\textbf{Filter Minimization (\fm{\struct{F}})}}
{A deterministic filter $\struct{F}$.}
{A deterministic filter $\struct{F}^{\star}$ with fewest states,~such~that:
\begin{itemize}[topsep=2pt, leftmargin=9pt,itemindent=2pt,itemsep=2pt]
\item [\cri{1}]\hspace*{-2pt}any sequence which can be traced on $\struct{F}$ can also be traced on $\struct{F}^{\star}$;
\item [\cri{2}]\hspace*{-2pt}the outputs they produce on any of those sequences are identical.
\end{itemize}
\vspace*{-12pt}
\phantom{.}
}

Solving this problem requires some minimally-sized filter~$\struct{F}^{\star}$ that is functionally equivalent to $\struct{F}$, where the 
notion of equivalence\,---called \emph{output simulation}---\,needs only 
criteria~\cri{1} and~\cri{2} to be met. For a formal definition of output simulation, see~\cite[Definition~5, pg.~93]{zhang20cover}.

\begin{lemma}[\cite{okane17concise}]
\label{lem:FMhardness}
The problem \fmp is \nphard.
\end{lemma}
\shortvspace{-0.75ex}

Recently, in giving a minimization algorithm\,\cite{zhang20cover}, 
\fmp  was shown to be
equivalent to vertex covering when 
the valid coverings satisfy a set of auxiliary constraints.
These constraints, denoted $\zipcoll{Z}$,  are termed \emph{zipper} constraints as they may cause long chains of vertices to be `pulled together' incrementally. 
First, we describe this vertex covering problem in and of itself. Next, this abstract problem will be connected back to filters through the notion of compatibility.

\ourproblemwithnote{\textbf{Minimum Zipped Clique Cover} (\mzcc{G}{\zipcoll{Z}})}
{A graph $G=(V,E)$ and 
a collection of ordered pairs of $G$'s edges
$\zipcoll{Z}=\{\zc{U_1}{V_1}{{y_1}},
\zc{U_2}{V_2}{y_2},
\dots,
\zc{U_m}{V_m}{y_m}\}$, where
\mbox{$U_i, V_i \in E$}.
}
{Minimum cardinality clique cover $\cover{K}$ such that:
\begin{itemize}[topsep=2pt,leftmargin=9pt,itemindent=2pt,parsep=6pt,itemsep=0pt]
\item [\cri{1}]\hspace*{-2pt}$\bigcup_{K_i \in \cover{K}} K_i = V$\!, 
 with each $K_i$ forming a clique on~$G$;
\item [\cri{2}]\hspace*{-2pt}$\forall K_i \in \cover{K}$, if there is some $\ell$ such that $U_\ell \subseteq K_i$, then some \mbox{$K_j \in \cover{K}$} must have $K_j \supseteq V_\ell$.
\end{itemize}
\vspace*{-8pt}
\phantom{.}
}
{
\hspace*{-2.5pt}(This is a special case of \mzccp in \cite{zhang20cover} but will suffice, see discussion in footnote~\ref{fn:svo}.)
}
\addtocounter{footnote}{1}
\footnotetext{After examining filters like those here, the later sections of that paper go further by studying a generalization in which function $c$ may be a relation. Complications arising from that generalization will not be discussed herein.\label{fn:svo}}



In bridging filters and covers, the key idea
is that certain sets of 
states in a filter can be identified as candidates to merge together, and such
`mergability' can be expressed as a graph. The process of forming covers of this graph identifies states to consolidate and, accordingly, minimal covers yield 
small filters. The first technical detail concerns this graph and states that are candidates to be merged:

\begin{definition}[extensions\gobble{ and }{/}compatibility]
\label{def:compat}
For a state $v$ of filter $\struct{F}$, \gobble{we will }use $\extension{\struct{F}}{v}$ to
denote the set of observation sequences, or \defemp{extensions}, that can be traced starting from~$v$.  
\gobblexor{Two states}{States} $v$ and $w$
are \defemp{compatible} \gobble{with each other }if their outputs agree on 
$\extension{\struct{F}}{v} \cap \extension{\struct{F}}{w}$, their
common extensions. In such cases, we \gobble{will }write $v\compatible w$.
The \defemp{compatibility graph} $G_\struct{F}$ possesses edges between states
if and only if
 they are compatible.
\end{definition}

But simply building a minimal cover on $G_\struct{F}$ is not enough because covers 
may merge some elements which, when
transformed into a filter,
produce nondeterminism.
The core obstruction is when 
a fork is created, as when
two
compatible states are merged, both
of which have outgoing edges bearing identical labels, but whose destinations
differ. 
To enforce determinism, we
use constraints to
forbid forking and require mergers
to flow downwards. 
\gobblexor{The following specifies such a constraint:}{See the following:}

\begin{definition}[determinism-enforcing zipper constraints]
\label{def:zipper_constraint}
Given a pair of compatible states $U=\{u_1, u_2\}$ in $\struct{F}$ and their $y$-children, $V=\{\tau(u_1, y), \tau(u_2, y)\}$,
then the  ordered pair $\zc{U}{V}{y}$ is a \emph{determinism-enforcing zipper constraint} of~$\struct{F}$. 
\end{definition}
A zipper constraint $\zc{U}{V}{y}$ is satisfied by a clique cover if $U$ is not covered in a clique, or both $U$ and $V$ are covered in cliques. (This is criterion  \cri{2} for \mzccp.)
For filters, in other words, if the states in $U$ are to be merged (or consolidated) then
the downstream states, in $V$, must be as well.
The collection of all  determinism-enforcing 
zipper constraints for a filter 
$\struct{F}$ is denoted $\zipcoll{Z}_\struct{F}$.


\begin{figure}[t!]
\centering
\includegraphics[width=0.75\linewidth]{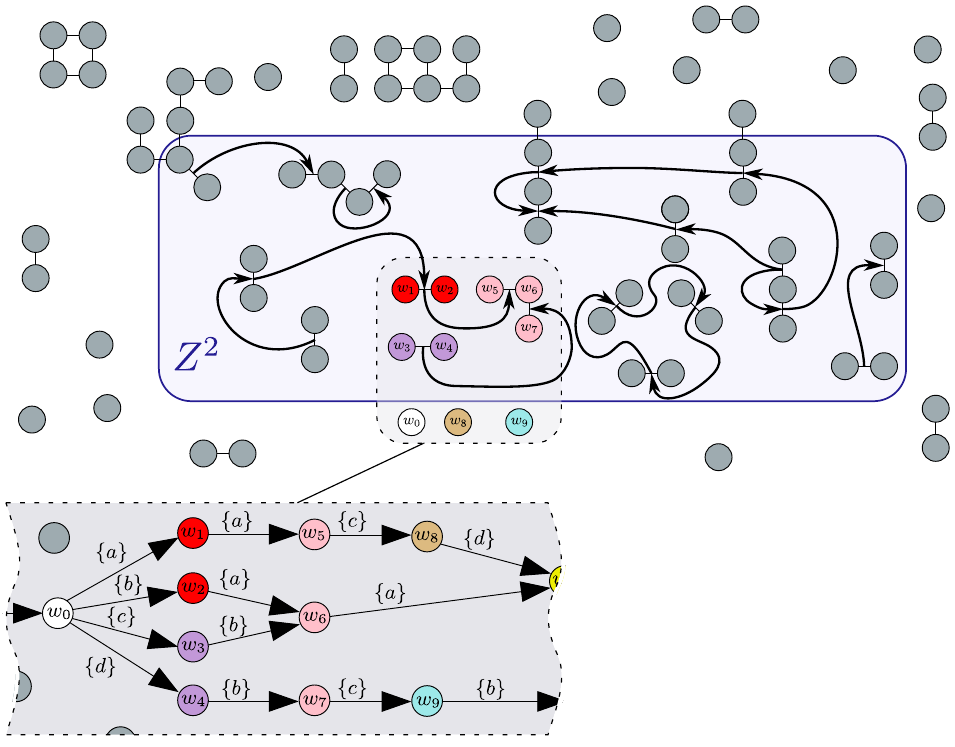}
\hspace*{7cm}
\begin{minipage}{0.35\textwidth}
\vspace*{-2.2cm}
\caption{\label{fig:filter_to_graph}An illustration showing 
(partially) a filter $\struct{F}$ (left inset)
leading to a compatibility graph 
$G_\struct{F}$ (above).
The $\Zpairs$ set is also shown.}
\end{minipage}
\vspace*{-1.0cm}
\end{figure}

In summary: 
for filter $\struct{F}$, the collection of zipper constraints 
$\zipcoll{Z}_\struct{F}$ 
that ensures the desired result will be a deterministic filter 
can be constructed directly as follows.
For $u\compatible v$, a pair of compatible states in~$\struct{F}$,
use $\tau(\cdot,y)$ to trace forward under observation~$y$; if the 
\mbox{$y$-children} thus
obtained are distinct, we form zipper constraint
$(\{u,v\},\{\tau(u, y), \tau(v, y)\})$ to ensure that if $u$ and $v$ are
merged (by occupying some set in the cover together), their \mbox{$y$-children} will be as well.
Construct~$\zipcoll{Z}_\struct{F}$ by collecting all such constraints for all compatible pairs, using every observation.
A cartoon illustrating the result of this procedure is shown 
in Figure~\ref{fig:filter_to_graph}: a snippet of the filter $\struct{F}$ appears at bottom left; its undirected compatibility graph, $G_\struct{F}$, appears above; zipper constraints are shown as the directed edges on the undirected compatibility edges.
Both $G_\struct{F}$ and $\zipcoll{Z}_\struct{F}$ are clearly polynomial in the size of~$\struct{F}$.
Then, a minimizer of $\struct{F}$ can be obtained from the solution to the minimum
zipped vertex cover problem, \mzcc{G_\struct{F}}{\zipcoll{Z}_\struct{F}}: 
\begin{lemma}[\cite{zhang20cover}]
Any \fm{\struct{F}}\!\! can be converted to an \mzcc{G_\struct{F}}{\zipcoll{Z}_\struct{F}} in polynomial time; hence \mzccp is \nphard.
\label{lm:fm_eqv}
\end{lemma}
Though we skip the details, the proof in~\cite{zhang20cover} of the preceding lemma also gives an efficient way to construct a deterministic
filter from the minimum cardinality clique cover.

As a final point on the connection of these two problems,  
combinatorial filters generate `most' graphs. Specifically, 
\cite{zhang23general} proves that  some filter $\struct{F}$ realizes $G$ as its constraint graph $G_{\struct{F}}$
if and only if either: (1) the graph $G$ has at least two connected components, or (2) $G$ is a complete graph.

\bigskip


In any graph $G=(V,E)$, 
we refer to the neighbors of a vertex $v\in V$
by set {$\neighbor{G}{v} \coloneqq \{ w\in V\;|\;(w,v) \in E\} \cup \{v\}$}. 
Note that we explicitly include $v$ 
in its own neighborhood.

\begin{definition}[\cn\footnotemark]
\label{def:nc}
A pair of vertices $\pr{v}{w}$ in some
graph $G=(V,E)$ have \defemp{\cn} 
if and only if we have either
\mbox{$\neighbor{G}{v}\subseteq \neighbor{G}{w}$} 
or  $\neighbor{G}{v}\supseteq \neighbor{G}{w}$.
\end{definition}

\footnotetext{This generalizes a concept first introduced in \cite[Definition~25]{zhang23general}.}

We will use the following recent result of Ullah:

\begin{lemma}[From \cite{ullah22structural}]
\label{existing:fpt}
Given a graph $G$ with $n$ vertices, let $\beta$ be the size of the largest clique in $G$, 
and let the number of cliques in the minimum clique cover be $m$, then there is 
an algorithm that computes the minimum clique cover in
$2^{\beta m \log m}\, n^{O(1)}$.
\end{lemma}
\ifmoveprooftoend
\begin{proof}
This follows from Theorem~1.8 of  \cite[pg.~4]{ullah22structural}, and the algorithms therein:
in his notation, we solve a VCC problem, which is possible
via the LRCC problem with its three parameters
set to $G$, $k = m$, $E^*=\emptyset$. 
\end{proof}
\else
\fi

\vspace*{-6pt}
\section{Zippers and prescriptions}
\vspace*{-4pt}

For a zipper constraint collection 
$\zipcoll{Z}=\{\zc{U_1}{V_1}{y_1}, \zc{U_2}{V_2}{y_2}, \dots, \zc{U_m}{V_m}{y_m}\}$, let
$\Zpairs = \{U_1, U_2, \dots, U_m, V_1, V_2, \dots, V_m\}$, i.e., 
the unordered pairs of vertices (or edges)
appearing within  collection $\zipcoll{Z}$.
We will write $P \up Q$ if and only if there exists a zipper constraint $\zc{P}{Q}{y} \in \zipcoll{Z}$. 
We define  $\ZpairsS \subseteq \Zpairs$ to be the 
set 
$\ZpairsS \coloneqq \{P \subseteq V(\struct{F})\;|\; \exists Q \in \Zpairs,  P \up Q\}$, i.e., the 
unordered pairs 
appearing as  sources in the zipper constraint collection $\zipcoll{Z}$.
Similarly, let $\ZpairsT \subseteq \Zpairs$ be those pairs  appearing as potential targets for enforced merging within the zipper constraints, i.e.,
$\ZpairsT \coloneqq \{Q \subseteq V(\struct{F})\;|\; \exists P \in \Zpairs,  P \up Q\}$.
By construction 
$\Zpairs = \ZpairsS \cup \ZpairsT$, and, in general, $\ZpairsS \cap \ZpairsT \neq \emptyset$. 

\begin{definition}
Given $\zipcoll{Z}$, the pair $P_b \in \Zpairs$ is \defemp{downstream} from pair 
$P_a \in \Zpairs$, written $P_a \strictupstream P_b$,
if $P_a \up P_b$, or  if $P_a \up P_c$
for some $P_c \in \Zpairs$ with  $P_c \strictupstream P_b$. 
\end{definition}

\subsection{Prescriptions}

To tackle the \mzccp problem, we search for covers subject 
to a rule stating that some specific pairs must be merged, while others must never be.
The idea is to make and fix choices for a subset of the pairs involved 
in the zipper constraint collection
 so that
this prescription respects the zipper constraints for elements in the collection. We will denote the collection via
set $\D$, defined next.

\begin{definition}
With zipper constraint collection $\zipcoll{Z}$
given some set of pairs \mbox{$\D \subseteq \Zpairs$}, 
a \defemp{prescription on $\D$} is a subset of pairs
$\Selon{S}{\D} \subseteq \D$.
Prescription $\Selon{S}{\D}$ on $\D$ is termed \defemp{downstream enabled} if and only if
$(P_a \in \Selon{S}{\D}) \wedge (P_a \strictupstream P_b) \implies 
(P_b \in \Selon{S}{\D}) \vee (P_b \not\in \D)$.
\end{definition}

The elements in $\Selon{S}{\D}$ are called
the \onpairs; those in  
$\D \setminus \Selon{S}{\D}$ are the \offpairs, which we write as $\Seloff{S}{\D}$. 
A prescription is silent about elements outside $\D$ (the mnemonic for \D being `domain' --- elements within the domain are prescribed as either being \on or \off; elements outside the domain have no prescription).
See Figure~\ref{fig:prescription}.
(Note that since $\D \subseteq \Zpairs$, both the 
\onpairs and \offpairs are from the set of zipper constraints, hence are edges within the graph in the 
\mzccp problem.)
If $\Selon{S}{\D} = \{P_1, P_2, \dots, P_m\}$
is a prescription, then it will be used to
require that the \onpairs be merged, while the \offpairs are prohibited from being merged. 
The idea is to ensure that a cover is produced that
respects the  prescription:

\begin{figure}[b!]
\centering
\begin{minipage}{0.6\textwidth}
\includegraphics[width=\linewidth]{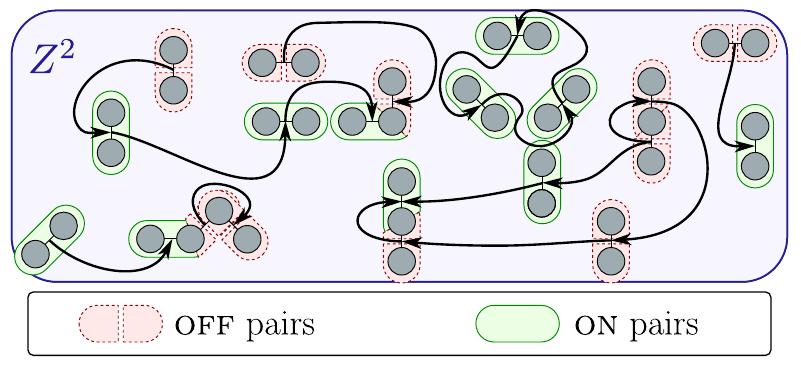}
\end{minipage}
\begin{minipage}{0.38\textwidth}
\caption{An example downstream-enabled prescription for 
the pairs in $\D = \Zpairs$ corresponding to the
instance in Figure~\ref{fig:filter_to_graph}.
Dashed red and solid green outlines represent states in $\struct{F}$ to be
split and merged, respectively.
\label{fig:prescription}
}
\end{minipage}
\end{figure}

\begin{definition}[Faithfulness]
\label{def:faithfull}
Let graph $G$ and 
the  collection of 
zipper constraints 
$\zipcoll{Z}$ be given.
For some domain $\D \subseteq \Zpairs$, 
a cover $\cover{K} = \{K_1, K_2,\allowbreak\dots, K_m\}$ 
of $G$
is \emph{faithful to}
prescription 
$\Selon{S}{\D} = \{P_1, P_2, \dots, P_n\}$
if and only if:
\begin{enumerate}
\item \label{def:faithful-on} For every $P_i \in \Selon{S}{\D}$ there exists 
some clique $K_\ell \in \cover{K}$
where $P_i \subseteq K_\ell$;
\item \label{def:faithful-off} For every $P_j \in \Seloff{S}{\D}$ 
there is no 
clique $K_\ell \in \cover{K}$ such that 
$P_j \subseteq K_\ell$.
\end{enumerate}
\end{definition}

We will achieve this via modification of a compatibility
graph that will make \on and \off sets 
enforce and prohibit mergers.
The modification of the graph is given through a series of set constructions next,
before showing (in Lemma~\ref{lem:faith}) that this can be used as desired.

\subsection{Enumerating downstream enabled prescriptions}

A pair graph is a directed graph whose vertices are pairs from $\D$. We will write $\prG{G}$, where vertices 
$\prG{V}(\prG{G}) \subseteq \D$ and edges
$\prG{E}(\prG{G}) \subseteq \prG{V}(\prG{G}) \times \prG{V}(\prG{G})$.
Let $\prG{G} - [R]$ denote the
subgraph \prG{G'} with vertices $R$ removed, i.e., vertices
$\prG{V}(\prG{G'})=\prG{V}(\prG{G}) \setminus R$, and 
also edges
$\prG{E}(\prG{G'}) = \left(\prG{V}(\prG{G'}) \times  \prG{V}(\prG{G'}) \right) \cap \prG{E}(\prG{G})$.
Next, we define the collection of up and downstream pairs within a given pair graph:
$\upSet{\prG{G}}{P}\!=\! \{ P_u\in\prG{V}(\prG{G})\!\mid\! P_u \strictupstream  P \}$;
$\downSet{\prG{G}}{P}\!=\! \{P_d\in\prG{V}(\prG{G})\!\mid\! P \strictupstream  P_d \}.$
\vspace*{1.1ex}

\begin{algorithm}
{
\caption{\enum(\prG{G})\label{algo:enum}}
\renewcommand{\algorithmicrequire}{\textbf{Input:}}
\renewcommand{\algorithmicensure}{\textbf{Output:}}
\begin{algorithmic}[1]
\REQUIRE{Pair graph \prG{G}}
\ENSURE{All prescriptions that are downstream enabled}\\[6pt]

\IF {$\prG{V}(\prG{G}) = \emptyset$}
\RETURN $\;\emptyset$
\gobblexor{\ELSE}{\ENDIF}
\STATE Let $P \in \prG{V}(\prG{G})$ be an arbitrary pair where \mbox{$\forall Q \in  \prG{V}(\prG{G})$}, $|\downSet{\prG{G}}{P}| \geq |\downSet{\prG{G}}{Q}|$ 
\RETURN $\;\;\enum\big(\prG{G} - [\upSet{\prG{G}}{P}]\big) \bigcup\Big\{\!\{P\}\cup \downSet{\prG{G}}{P} \cup S\,\big|\,S \in \enum\big(\prG{G} - [\downSet{\prG{G}}{P}]\big)\!\Big\}$\\[-1pt] \label{line:non-triv}
\gobblexor{\ENDIF}{}
\end{algorithmic}
}
\end{algorithm}
\vspace*{-1ex}

To generate all downstream enabled prescriptions, we 
invoke \enum(\prG{G}(\D)) in Algorithm~\ref{algo:enum},
 where $\prG{G}(\D)$ is the pair graph having  vertices $\D$, and an edge from pair $P_s$ to $P_d$
if and only if $P_s \up P_d$.
In the return statement in line~\ref{line:non-triv}, the first set corresponds to prescriptions 
where $P$ is being turned \off, while the second set has $P$ being  turned \on.
Note that because cycles
are possible between zipper constraints,  $\upSet{\prG{G}}{P} \cap 
\downSet{\prG{G}}{P} \neq \emptyset$, in general. 
The presence of cycles reduces the number of
prescriptions to enumerate.
When $P$ appears within a cycle, 
if $P$ is \off, then all pairs in the cycle have to be \off too;
if $P$ in \on, then they are all also \on as well. 

To dispatch with cycles,
define a relation on pairs in $\D$: let $P \cyceq Q$ if and only if 
$P = Q$ or $P \strictupstream Q \strictupstream P$. Further, $\cyceq$ is an equivalence 
relation  and
the set of equivalence classes $\D /\!\!\!\cyceq$ can be partially ordered by lifting the  $\strictupstream$ relation.
Then the total number of  downstream enabled prescriptions, $|\enum(\prG{G})|$, is bounded by 
$((\ell + 1)+1)^{\omega} = 2^{\omega \lg(2 + \ell)}$
with
$\ell$ being the height (i.e., the length of the longest chain) and $\omega$ width (i.e., size of the largest anti-chain) 
of $(\D /\!\!\cyceq, \strictupstream)$ respectively.
See Figure~\ref{fig:poset} above.

\begin{figure}[t]
\centering
\includegraphics[width=0.70\linewidth]{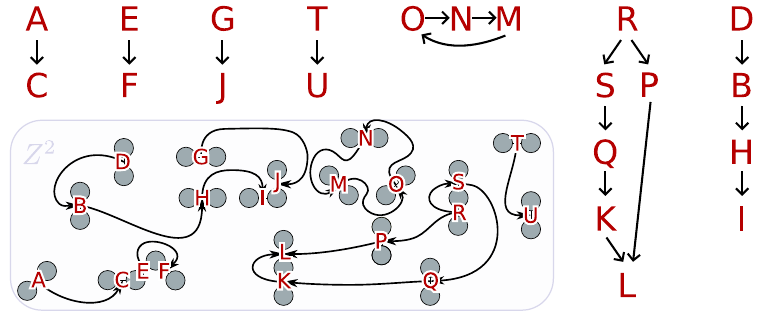}
\vspace*{-1ex}
\caption{Partial order of $\strictupstream$.
Here,
height 
$\ell = 4$ 
and width 
$\omega=8$.\label{fig:poset}}
\vspace*{-2ex}
\end{figure}

\smallskip

As the enumeration is of 
 downstream-enabled prescriptions from $\Zpairs$,
it is worth noting, firstly, that 
 $\Zpairs$ will often be much smaller than the size of the 
input filter ($|\Zpairs| \leq |V(\struct{F})|$). 
Second, $\D$ may be a proper subset of $\Zpairs$ (see Section~\ref{sec:repairable}, where $|\D| \leq |\Zpairs|$).
Third, the number of downstream-enabled prescriptions 
for $\D$ will often be much smaller than $2^{|\D|}$,
i.e., $\omega \ll |\D|$,
owing to both the reduction obtained from cycles 
($|\D/\!\!\cyceq\!| \leq |\D|$),
and ordering structure inherited from 
the sequential zipping that arises from
the filter (e.g., as $\ell$ increases, 
worst-case  $\omega$ decreases).

\section{Graph Augmentation}
Downstream-enabled prescriptions are effective at encoding
choices that are imperative to satisfaction of the zipper constraints, without specifying a full clique cover. 
Given such a prescription, first we augment the compatibility graph by incorporating the prescription. Thereafter, as the zipper constraints are no longer a concern, we can focus on the remaining minimum clique cover problem on the augmented graph. Finally we transport the solution from the augmented graph back to the original compatibility graph.

\subsection{Augmenting the constraint graph}


\begin{construction}[Augmented Graph $G^+$]
\label{const:augmentation}
If $G = (V,E)$, then construct $G^+(\Selon{S}{\D}) = \left(V^+(\Selon{S}{\D}), E^+(\Selon{S}{\D})\right)$ with 
$V^+(\Selon{S}{\D}) \coloneqq V_{\prior} \cup V_{\extra}(\Selon{S}{\D})$ where
\begin{itemize}
\item[] $V_{\prior} \coloneqq 
    \big\{  \newvertexS{u}\;\big|\;
    u \in V \big\}$,
 
\item[] $V_{\extra}(\Selon{S}{\D}) \coloneqq 
    \big\{  \newvertex{u}{w}\;\big|\;
    \pr{u}{w} \in \Selon{S}{\D} \big\}$,

\item[] $E^+(\Selon{S}{\D})  \coloneqq
\Big\{  \pr{\newvertexSet{A}}{\newvertexSet{B}}
\in V^+(\Selon{S}{\D})\times V^+(\Selon{S}{\D})
\;\Big|\;
A \cup B\text{ form a clique in}\phantom{xxx} \allowbreak 
\phantom{.}\hfill\text{graph }G'=(V, E \setminus \Seloff{S}{\D})
\Big\}$.
\end{itemize}
\end{construction}


The vertices in $V_\prior$ are simply re-named copies of those in $V$. 
The set $V_{\extra}$ introduces new vertices for those pairs in $\D$ which have been turned 
\on.
The definition of the edge set adds edges to ensure that the new vertices will be seen as mutually compatible when there is no obstruction to compatibility from within the original graph.
\bigskip 

\begin{property} 
\label{prop:vertex-sets-quantify-all}
Let vertices $\newvertexSet{A},\newvertexSet{B} \in V^+(\Selon{S}{\D})$,  $\newvertexSet{A} \neq \newvertexSet{B}$
then
\begin{equation*}
\{\newvertexSet{A},\newvertexSet{B}\}
 \in  E^+(\Selon{S}{\D}) \iff 
 \forall
u \in A, \forall w \in B, u \neq w, \{\newvertexS{u}, \newvertexS{w}\} \in  E^+(\Selon{S}{\D}).
\end{equation*}



\end{property} 
\ifmoveprooftoend
\begin{proof}
When $A$ and $B$ are singletons,
i.e., $\newvertexSet{A},\newvertexSet{B}\in V\prior$, 
the two sides of the if and only if are identical statements.
\begin{itemize}

\item [$\impliedby$] 
The given antecedents and fact that $\newvertexSet{A},\newvertexSet{B} \in V^+(\Selon{S}{\D})$ are exactly the conditions in the construction of 
$E^+(\Selon{S}{\D})$, hence $\{\newvertexSet{A},\newvertexSet{B}\}\in E^+(\Selon{S}{\D})$.

\item [$\implies$] 
Suppose
$\{\newvertexSet{A},\newvertexSet{B}\} \in  E^+(\Selon{S}{\D})$
but there is some $u \in A$ and $w \in B$ with
$\{\newvertexS{u},\newvertexS{w}\} \not\in  E^+(\Selon{S}{\D})$.
Since $u \neq w$, hence
$\{u,w\} \not\in E$ or
$\{u,w\} \in 
\Seloff{S}{\D}$;
both contradict the 
supposition that
$\{\newvertexSet{A},\newvertexSet{B}\} \in  E^+(\Selon{S}{\D})$. 
\end{itemize}
\end{proof}
\else
\fi

\subsection{Relating graphs and their augmentations}

Next, we consider two operations which connect 
vertices in the original graph with those in its 
augmented graph.

\begin{definition}[Distillation]
\label{def:distillation}
Suppose a graph  $G = (V,E)$ and its augmentation $G^+(\Selon{S}{\D})$ based on 
prescription $\Selon{S}{\D} \subseteq \D$ is given.
The set of vertices of the augmented graph, $S^+ = \{\newvertexSet{A_1},\newvertexSet{A_2},\dots, \newvertexSet{A_{n}}\} \subseteq V^+(\Selon{S}{\D})$, may be \defemp{distilled} to obtain a set of vertices in the original graph:
$S \coloneqq  A_1 \cup  A_2 \cup \dots \cup A_{n}.$
\end{definition}

In the preceding, when $S$ is obtained from $S^+$ in this way, we will also 
refer to it as $S^+$'s distillate.
Further, 
we will also talk of the distillate of
a collection of sets $\{S_1^+, S_2^+, \dots, S_n^+\}$,
as the collection obtained by applying Definition~\ref{def:distillation} to each $S_i^+$, each
yielding their respective $S_i$.
In the particular uses of this concept which follow we will be interested in distilling collections that are covers.
The next property shows that 
distillation preserves cliqueness,
while transporting a structure from graph $G^+$ back to $G$. 

\begin{property}
\label{prop:distilate-cliques}
For a graph  $G = (V,E)$ and its augmentation $G^+(\Selon{S}{\D})$ based on 
prescription $\Selon{S}{\D} \subseteq \D$, 
suppose
$S^+ \subseteq V^+(\Selon{S}{\D})$ produces 
$S \subseteq V$ when distilled.
Then $S^+$ is a clique in  $G^+(\Selon{S}{\D})$ 
if and only if $S$ is a clique in  $G' = (V, E \setminus (\Seloff{S}{\D}))$.
\end{property} 
\ifmoveprooftoend
\begin{proof}
$\impliedby:$  
Let $S^+ = \{\newvertexSet{A_1},\newvertexSet{A_2},\dots, \newvertexSet{A_{n}}\}$.
According to Construction~\ref{const:augmentation}, for every $\newvertexSet{A_i}$ and $\newvertexSet{A_j}$,  with $i \neq j$, 
we know that $A_i \cup A_j$ form a clique in $G'$. But $A_i \cup A_j \subseteq S$; since $S$ is a clique,
$A_i \cup A_j$ is a clique, and hence there is an edge in $E^+(\Selon{S}{\D})$ connecting 
$\newvertexSet{A_i}$ and $\newvertexSet{A_j}$.

$\implies:$ 
Suppose $S$ was not a clique, then there are distinct vertices $u,w \in S$ that have no connecting edge in $G'$.
But $S$ being the distillate of $S^+$ means that 
there is  some $\newvertexSet{A_u} \in S^+$ with $u \in A_u$, and
some $\newvertexSet{A_w} \in S^+$ with $w \in A_w$.  
If $A_u = A_w =\pr{u}{w}$,
then  $\newvertex{u}{w} \in V_{\extra}(\Selon{S}{\D})$,
but that implies $\pr{u}{w} \in \Selon{S}{\D} \subseteq E \setminus \Seloff{S}{\D}$,  a contradiction.
Otherwise, $A_u \neq A_v$ but this is also a contradiction for we know $S^+$ is a clique, so
there is an edge between
$\newvertexSet{A_u}$ and  $\newvertexSet{A_w}$, but
Construction~\ref{const:augmentation}, thus, requires
an edge between $u$ and $w$ in $G'$.
\end{proof}
\else
\fi

The concept in the following definition is a sort of counterpoint to 
that of Definition~\ref{def:distillation}.

\begin{definition}[Expansion]
Suppose a graph  $G = (V,E)$ and its augmentation $G^+(\Selon{S}{\D})$ based on 
prescription $\Selon{S}{\D} \subseteq \D$ is given.
A set of vertices $S \subseteq V$ can be 
\emph{expanded} to
give elements 
of $V^+(\Selon{S}{\D})$,
i.e., vertices in 
 $G^+(\Selon{S}{\D})$:
\begin{equation}
\label{eq:undistill}
\expand{S} \coloneqq \{ \newvertexSet{A} \in V^+(\Selon{S}{\D}) \mid A\subseteq S\}.
\end{equation}
\end{definition}

(Notice the subtlety that binding elements to within $V^+(\Selon{S}{\D})$ ensures $A$ will be singletons or pairs.)
Observe that if  $S^+=\expand{S}$, 
i.e. $S$ expands to $S^+$, 
then the distillation of $S^+$ is $S$\,---\,this is proved as the first part of the next property.

\begin{property} 
\label{prop:un-distilate-cliques}
Given a graph  $G = (V,E)$ and its augmentation $G^+(\Selon{S}{\D})$ based on 
prescription $\Selon{S}{\D} \subseteq \D$. 
If $\cover{K}=\{K_1, \dots,K_m\}$ is a clique cover of $G$ faithful to  $\Selon{S}{\D}$ then, 
collecting the expanded sets in $\cover{K}^+=\{\expand{K} \mid K\in\cover{K}\}$, the  collection $\cover{K}^+$ is a clique cover on $G^+(\Selon{S}{\D})$. 
\end{property} 
\ifmoveprooftoend
\begin{proof}
To start we establish that the distillate of $\expand{S}$ is $S$. 
Let $S'$ be the  distillate of $\expand{S}$;
we show equality via two subset statements.
First, $S \subseteq S'$ as 
if $x \in S$ then $\newvertexS{x} \in \expand{S}$, and hence is in $S'$.
For the reverse, $S' \subseteq S$:
if $y \in S'$ then either $\newvertexS{y} \in \expand{S}$,
or $\newvertex{y}{w} \in \expand{S}$ for some $w$; but then $\{y\} \subseteq S$ or $\{y,w\} \subseteq S$, respectively.

Hence, for any $K_i^+ \in \cover{K}^+$, where $K_i^+ = \expand{K_i}$, 
we know $K_i$ is the distillate of $K_i^+$.
Then,  via Property~\ref{prop:distilate-cliques}, each $K_i^+$ is a clique 
in $G^+(\Selon{S}{\D})$ because $K_i$ is a clique in $G'=(V, E \setminus (\Seloff{S}{\D}))$. (This latter is a consequence of faithfulness: owing to requirement~\ref{def:faithful-off} in Definition~\ref{def:faithfull}.)

Furthermore, the collection $\cover{K}^+$ cannot omit to cover any
$\newvertexSet{A} \in V^+(\Selon{S}{\D})$ because: 
($i$) each element of $\newvertexS{v} \in V_{\prior}$ corresponds to a vertex
$v \in V$, and $v$ is covered by $\cover{K}$;
($ii$) elements $\newvertex{u}{w} \in V_{\extra}(\Selon{S}{\D})$
come from  $\pr{u}{w} \in \Selon{S}{\D}$, and 
faithfulness of $\cover{K}$ means, via
requirement~\ref{def:faithful-on} in Definition~\ref{def:faithfull},
that some $K_i \in \cover{K}$ will cover
$\pr{u}{w}$, and hence $\expand{K_i}$ covers $\newvertex{u}{w}$.
\end{proof}
\else
\fi

\section{Connecting covers, prescriptions, and constraints}
\label{section:covers}

We now have the machinery in place to present a
 useful lemma. This will show that the
augmented graph,  recalling Definition~\ref{def:faithfull} for faithfulness, will
yield covers that adhere to the prescription.

\begin{lemma}[Faithful constraint satisfaction]
\label{lem:faith}
Given a graph $G$ and an associated $\Zpairs$ from zipper constraint collection $\zipcoll{Z}$, let $\D \subseteq \Zpairs$. Then,
suppose we have 
some downstream-enabled prescription
$\Selon{S}{\D} = \{P_1, P_2, \dots, P_n\}$.
If  $\cover{K}^+$ is 
a cover of $G^{+}(\Selon{S}{\D})$, then 
there is a cover $\cover{K}$ of $G$,
the distillation of $\cover{K}^+$,
with $|\cover{K}| \leq |\cover{K}^{+}|$,
such that:
\begin{itemize}
\item Cover $\cover{K}$ is faithful to 
$\Selon{S}{\D}$.
\item Also, $\cover{K}$ is a cover of $G$ which satisfies all those zipper constraints strictly
interior to $\D$, namely those
$\big\{\zc{U_1}{W_1}{y_1},\allowbreak\zc{U_2}{W_2}{y_2},\dots,\zc{U_t}{W_t}{y_t}\big\} \subseteq  \zipcoll{Z}$ where $U_1, W_1, U_2, W_2,\dots, U_t, W_t \in \D$.

\end{itemize}
\end{lemma}
\ifmoveprooftoend
\begin{proof}
Let $\cover{K}^+ = \{K^+_1, K^+_2, \dots, K^+_m\}$,
where each $K^+_i \subseteq   V^+(\Selon{S}{\D})$. 
Applying Definition~\ref{def:distillation} to distill each $K^+_i$ to yield a $K_i$, we obtain
$\cover{K} = \{K_1, K_2, \dots,\allowbreak K_m\}$.
Via Property~\ref{prop:distilate-cliques}, each $K_i$ is a clique, and
since every vertex in 
$V_{\prior} \subseteq V^+(\Selon{S}{\D})$ is covered by some $K^+_i$, 
$\cover{K}$ is a cover of $G$.
Clearly, $|\cover{K}| \leq |\cover{K}^+|$, and equality only fails to hold when distinct $K^+_r$ and $K^+_s$ give rise to  $K_r = K_s$.

First, we show that the faithfulness of $\cover{K}$ follows from 
Construction~\ref{const:augmentation} and
the process of distillation.
{Criterion 1:}
for any $\pr{v_1}{v_2} \in \Selon{S}{\D}$ there is a vertex 
$\newvertex{v_1}{v_2}$ in  $V_\extra(\Selon{S}{\D})$ and
this vertex must be covered by some $K^+_j \in \cover{K}^+$, which then will have $v_1$ and $v_2$ covered by its 
distillate $K_j$.
{Criterion 2:}
if, instead, $\pr{v_1}{v_2} \in \Seloff{S}{\D}$ then  
the construction ensures that 
the distillate $\cover{K}$
of $\cover{K}^+$ 
will never have both $v_1$ and $v_2$ in the same clique:  suppose 
both $v_1$ and $v_2$ appear in clique
$K_j$, then there must be 
$\newvertexSet{A},\newvertexSet{B} \in  K^+_j$, where $\{v_1, v_2\} \in A \cup B$ (including the possibility that $A = B$).
Further, 
$E^+(\Selon{S}{\D})$ has 
no edge between  $\newvertexS{v_1}$ and  $\newvertexS{v_2}$ as 
$\Seloff{S}{\D}$ eliminates it
from both $E_{\cpy}(\Selon{S}{\D})$ and $E_{\extra}(\Selon{S}{\D})$.
But these two facts contradict
Property~\ref{prop:vertex-sets-quantify-all}.

Lastly, we show that $\cover{K}$ satisfies the subset of zipper constraints strictly interior to \D. 
Suppose that some zipper constraint 
$\zc{U_\ell}{W_\ell}{y_\ell} \in \zipcoll{Z}$ with \mbox{$U_\ell, W_\ell \in \D$} is violated in
$\cover{K}$.  
This means that the pair of vertices $U_\ell$ must appear within some clique (i.e., $U_\ell \subseteq K_i$) while $W_\ell$ does not 
(i.e., for all $K_j \in \cover{K}$, $W_\ell \not\subseteq K_j$).
As $\cover{K}$ is faithful and $K_i \in \cover{K}$, we know $U_\ell \in \Selon{S}{\D}$.
There are two cases for
$W_\ell \in \D$:
\begin{itemize}
\item [--] If $W_\ell = \pr{w_1}{w_2} \in \Selon{S}{\D}$, then 
there must be a corresponding vertex $\newvertex{w_1}{w_2} \in V^+(\Selon{S}{\D})$ from $V_\extra(\Selon{S}{\D})$. 
Since $\cover{K}^+$ is a cover, the $\newvertex{w_1}{w_2}$ is in at least one $K^+_j \in \cover{K}^+$. 
But then, following the distillation of $K^+_j$ into $K_j$, both $w_1$ and $w_2$ are within $K_j$, so $W_\ell \in K_j$, giving a contradiction.

\item [--] If $W_\ell \in \Seloff{S}{\D}$ then,
as $U_\ell\strictupstream W_\ell$, this contradicts
the fact that $\Selon{S}{\D}$ is  downstream enabled.
\end{itemize}%
\vspace*{-12pt}
\end{proof}
\else
\fi

Further, when cover $\cover{K}^+$ is minimal, then the 
preceding result can be 
strengthened, as we show next.

\begin{lemma}[Optimal faithful constraint satisfaction]
\label{lem:opt-faith}
Given the elements in 
Lemma~\ref{lem:faith},
if $\cover{K}^+$ is a \emph{minimal} clique cover of $G^{+}(\Selon{S}{\D})$, then, in addition 
to the properties in 
the previous lemma,
cover $\cover{K}$ of $G$
has:
\begin{itemize}
\item  $|\cover{K}| = |\cover{K}^{+}|$. 
\item There exists no clique cover  $\cover{K}_\hyp$ of $G$, faithful to  $\Selon{S}{\D}$, with  $|\cover{K}_\hyp| < |\cover{K}|$.
\end{itemize}
\end{lemma}
\ifmoveprooftoend
\begin{proof}
Suppose 
$|\cover{K}| \lneq |\cover{K}^+|$, but
this happens only when some
$K^+_r, K^+_s \in \cover{K}^+$, with 
$K^+_r \neq  K^+_s$, distill to $K_r = K_s$.
Were this the case, one may 
obtain a valid clique cover for $G^{+}(\Selon{S}{\D})$ 
by replacing the cliques $K^+_r$ and $K^+_s$ with the single set
$K^+_{rs} = K^+_r \cup K^+_s$. 

The union of identical unions (underlying Definition~\ref{def:distillation}) means that 
$K_{rs}^+$  distills to $K_r$ as well.
Applying Property~\ref{prop:distilate-cliques} (in `if' direction from $K^+_r$) means that $K_r$ is a clique.
And applying Property~\ref{prop:distilate-cliques} (in the `only if' direction from $K_r$) means that $K^+_{rs}$ is a clique.
Notice that 
no vertex will be uncovered in
$\big\{K^+_{rs}\big\} \cup \cover{K}^+ \setminus \big\{K^+_r, K^+_s\big\}$, 
hence we have obtained smaller clique cover than the minimal one.

For the second claim, suppose some $\cover{K}_\hyp$, $|\cover{K}_\hyp| < |\cover{K}|$, 
is faithful to $\Selon{S}{\D}$. 
Then Property~\ref{prop:un-distilate-cliques}
indicates that 
$\cover{K}_\hyp^+$ can be found
such that it is a clique cover of 
 $G^{+}(\Selon{S}{\D})$.
Moreover, then $|\cover{K}_\hyp^+| \leq  |\cover{K}_\hyp| < |\cover{K}| \leq |\cover{K}^+|$, hence
$|\cover{K}_\hyp^+| < |\cover{K}^+|$,
which is a contradiction since 
$\cover{K}^+$ is assumed to be a minimal cover of 
 $G^{+}(\Selon{S}{\D})$.
\end{proof}
\else
\fi

Combining Algorithm~\ref{algo:enum}
with the preceding results, 
and picking $\D = \Zpairs$,
we already have an \fpt-algorithm for \mzccp. 
For each  prescription $\mathbb{S}^{i} \in  \enum(\prG{G})$, 
one constructs
$G_\struct{F}^{+}(\mathbb{S}^{i})$, then uses an \fpt-algorithm to
solve that  classical minimum clique cover. 
As per Lemma~\ref{lem:opt-faith}, one distills that 
cover into a zipper-constraint--satisfying cover for $G$; the smallest such cover\,---across all $\mathbb{S}^{i}$s---\,will be a solution to the problem. 
(This claim requires proof, but becomes a special case of a later result, by taking $\D = \Zpairs$ in Theorem~\ref{thrm:Dis}.)
Next, an improved algorithm, which takes more care to
pick a (potentially) smaller $\D$ will be presented.

\section{Repairable constraints}
\label{sec:repairable}

We may be able to pick $\D$ as a strict subset of $\Zpairs$ if 
there are zipper constraints which, though they may be violated during the covering process, can be resolved thereafter. The next lemma, making use of Definition~\ref{def:nc}, will show this:

\begin{lemma}
\label{lem:repairR}
Given a graph $G$ and an associated $\Zpairs$ from zipper constraint collection $\zipcoll{Z}$,
let $\R \subseteq \Zpairs$ be a set of pairs such that
for every pair \pr{u}{w} in $\R$, $u$ and $w$ have \cn.  If $\uR \subseteq \R$ then
let $\srcSet{\uR} \coloneqq \{\zc{U_i}{W_i}{y_i} \in Z_\struct{F} \;|\; U_i\in \uR\}$ 
and $\destSet{\uR} \coloneqq \{\zc{U_j}{W_j}{y_j} \in Z_\struct{F} \;|\; W_j\in \uR\}$.
If $\cover{K}$ is a cover for  $G$, then there is a cover $\widetilde{\cover{K}}$, no larger than $\cover{K}$, such  that:
\begin{enumerate}

\item  $\widetilde{\cover{K}}$ is a clique cover of $G$ that satisfies the specific zipper constraints
$\srcSet{\uR} \cap \destSet{\uR}$.\label{item:repaired_inside}
\item If  
$\cover{K}$ satisfies the zipper constraints $C \subseteq \zipcoll{Z}$,
and 
$\widetilde{\cover{K}}$ satisfies the zipper constraints $\widetilde{C} \subseteq \zipcoll{Z}$,  and
 then
${C} \setminus \widetilde{C} \subseteq 
\srcSet{\uR} \setminus \destSet{\uR}$.\label{item:only_breaks_boundary}
\end{enumerate}
\end{lemma}

The intuition is that we can \emph{repair} $\cover{K}$, without increasing its size, to ensure that those zipper constraints wholly in $\uR$ will hold (item~\ref{item:repaired_inside}).
This process can have an unfortunate side-effect of breaking some zipper constraints which held formerly: but those are only the
zipper constraints that `depart' $\uR$, i.e., $\srcSet{\uR} \setminus \destSet{\uR}$ (item~\ref{item:only_breaks_boundary}).
\ifmoveprooftoend
\begin{proof}
Let $\cover{K} = \{K_1, K_2, \dots, K_m\}$.
Collect all the interior pairs $\interiorZ \coloneqq \{ W_i \in \ZpairsT \;|\; \zc{U_i}{W_i}{y_i} \in \srcSet{\uR} \cap \destSet{\uR} \}$. 
Starting with cover $\cover{K}\iterIdx{0} = \cover{K}$,
we iterate over collection $\interiorZ$ and modify the cover incrementally.
Form $\cover{K}\iterIdx{i}$ from  $\cover{K}\iterIdx{i-1}$ by taking the $i^{\rm th}$ pair $\pr{u}{w}$ from $\interiorZ$ and doing the following:
if $\neighbor{G}{u}\subseteq \neighbor{G}{w}$ then, first, copy those $K_\ell \in \cover{K}\iterIdx{i-1}$ not containing $u$ to $\cover{K}\iterIdx{i}$;
next, gather those $K_j \in \cover{K}\iterIdx{i-1}$ containing $u$ and place $K_j \cup \{w\}$ in $\cover{K}\iterIdx{i}$.
Otherwise, $\neighbor{G}{u}\supseteq \neighbor{G}{w}$, so do the same two operations but reverse the roles of $v$ and $w$.
Once this iteration has been completed, take $\widetilde{\cover{K}} = \cover{K}\iterIdx{|\interiorZ|}$.

In the construction above,  `coverness' must be preserved as the sets in $\cover{K}\iterIdx{i}$ only grow with each subsequent $i$. 
Because, when $w$ is added to a clique containing $v$, the former must already be compatible with all those vertices in the clique\,---via the \cn property---\,the set $K_j \cup \{w\}$ is a clique too. 
Also, every zipper constraint in  $\srcSet{\uR} \cap \destSet{\uR}$ is satisfied because every destination pair of each such zipper constraint will appear in some clique in $\widetilde{\cover{K}}$.
Further, $|\widetilde{\cover{K}}| \leq  m$.

For the second property, notice that in $\widetilde{\cover{K}}$ the only pairs that have changed are those in $\interiorZ$;
they have been altered by including some vertices in a common clique, where before they had been separated.
But this change can only alter the satisfaction of constraints for which those pairs act as sources, viz. $\srcSet{\uR}$. So
${C} \setminus \widetilde{C} \subseteq \srcSet{\uR}$. But, as just established, those elements in 
$\srcSet{\uR} \cap \destSet{\uR}$ are satisfied, so  ${C} \setminus \widetilde{C}$ cannot include pairs in $\destSet{\uR}$, thus the claim follows.
\end{proof}
\else
\fi

The preceding shows that zipper constraints with both ends in a set $\uR \subseteq \R$ which
possesses \cn, need not cause any trouble. 
Our prior discussion, using  downstream-enabled prescriptions, allows one to deal with constraints
entirely outside of $\R$. 
However, a remaining difficulty is that some constraints may straddle the two sets.
We put $\R$ aside briefly, returning to it again in 
Lemma~\ref{lemma:extension-of-zippercs-if-d-less-repairable}
and subsequent theorems, as we now introduce extra machinery for the liminal constraints.

Broadly speaking, the preceding shows that rather than taking $\D = \Zpairs$ we can avoid having to include the
comparable-neighborhood pairs in the enumeration. This idea is close to being correct, but we need to ensure $\D$ will handle the liminal
constraints correctly as well. To do this, 
the idea will be to expand the domain $\D$ to embrace some 
additional pairs. 
(The additional pairs are those, when interpreted back in the filter, whose merger or non-merger is entailed from the choices made in a given prescription on~$\D$.)

\begin{construction}[Prescription Boundary Inclusion]
\label{const:bound-inclusion}
Given a prescription $\Selon{S}{\D}$, we
modify it by increasing 
$\Selon{S}{\D}$ and $\Seloff{S}{\D}$.
This is achieved, firstly, by modifying its domain and then, secondly, by selecting some additional elements, which transforms
it into a new prescription.
To do so, define \gobble{two }sets:
\begin{enumerate}
\item 
Upstream vertex pairs of the 
\offpairs should be treated as if they were turned
\off too, i.e., prohibited from being in a clique together 
(the constraint cannot be satisfied otherwise);
let $\mathcal{F}^\off_\D(\Selon{S}{\D}) = \{P_a \in \Zpairs\setminus\D\;|\;\exists P_b\in \D \setminus \Selon{S}{\D},  P_a\strictupstream P_b\}$.
\item 
Downstream vertex pairs of the 
\onpairs should be turned
\on as well, i.e. must be in some clique together;  
let $\mathcal{F}^\on_\D(\Selon{S}{\D}) = \{P_c \in \Zpairs\setminus\D\;|\;\exists P_d\in \Selon{S}{\D},  P_d\strictupstream P_c\}$.
\end{enumerate}
Construct a derived prescription by expanding the domain and \onpairs as
\begin{equation*}
\eD \coloneqq \D \cup \mathcal{F}^\off_\D(\Selon{S}{\D}) 
\cup \mathcal{F}^\on_\D(\Selon{S}{\D}),\text{ and }
\eSelon{S}{\eD}\coloneqq \Selon{S}{\D} \cup \mathcal{F}^\on_\D(\Selon{S}{\D}), 
\end{equation*}
where the \onpairs have 
grown to include those in $\mathcal{F}^\on_\D$. And, as before, 
the \offpairs are those in  $\eD \setminus \eSelon{S}{\eD}$;
and the prescription is silent about the elements outside $\eD$.
\end{construction}

Caution: in  $\eSelon{S}{\eD}$ we have lightened the notation by eliding the 
dependence of $\eD$ on $\Selon{S}{\D}$. Care is warranted because $\eD$ cannot be constructed from $\D$ alone---different prescriptions will give different $\eD$s.

\begin{property}
\label{prop:downstream-preserved-under-boundary-inclusion}
If prescription $\Selon{S}{\D}$ is downstream enabled
then  $\eSelon{S}{\eD}$ is downstream enabled.
\end{property}
\begin{proof}
We prove a slightly stronger statement, which is that
$(P_a \in \eSelon{S}{\eD}) \wedge (P_a \strictupstream P_b)\implies (P_b\in \eSelon{S}{\eD})$.
Given that antecedent, we have
$P_a \in \Selon{S}{\D}$ or $P_a \in \mathcal{F}^\on_\D(\Selon{S}{\D})$,
since $\eSelon{S}{\eD} = \Selon{S}{\D} \cup \mathcal{F}^\on_\D(\Selon{S}{\D})$. Then we need to show
\begin{itemize}
\item [1.] $(P_a \in \Selon{S}{\D}) \wedge (P_a \strictupstream P_b)\implies (P_b\in \eSelon{S}{\eD})$ and
\item [2.] $(P_a \in \mathcal{F}^\on_\D(\Selon{S}{\D})) \wedge (P_a \strictupstream P_b)\implies (P_b\in \eSelon{S}{\eD})$.
\end{itemize}
For the first, $P_b$ is certainly in $\eSelon{S}{\eD}$: 
When $P_b\not\in \D$ then Construction~\ref{const:bound-inclusion} ensures
$P_b\in \mathcal{F}^\on_\D(\Selon{S}{\D})$;
alternatively, 
when $P_b\in \D$ then $P_b\in \Selon{S}{\D} \subseteq \eSelon{S}{\eD}$
due to the original prescription being downstream enabled.

For the second,  
the definition of $\mathcal{F}^\on_\D(\Selon{S}{\D})$ means there is 
some pair  $P_d \strictupstream P_a$ with $P_d \in \Selon{S}{\D}$.
Transitivity and $P_a \strictupstream P_b$ means 
$P_d \strictupstream P_b$.
Again, $P_b$ is certainly in $\eSelon{S}{\eD}$, using the argument
above but with $P_d$ fulfilling the role of $P_a$ before.
\end{proof}

(Note, as per the discussion at the end of Section~\ref{section:covers}, when
 $\D = \Zpairs$, 
then $\eD = \D = \Zpairs$, and
we have 
$\eSelon{S}{\eD} =
\eSelon{S}{\D} =
\eSelon{S}{\Zpairs}$ so long as 
$\eSelon{S}{\D}$ was downstream enabled.)

\section{Main result: \fpt algorithm}

The paper's key algorithm just assembles all the pieces presented up to this point; it appears in Algorithm~\ref{algo:fpt}. The following lemma and theorem provide its correctness, while the
final corollary gives the parameterized running-time.

\floatname{algorithm}{Algorithm}
\vspace*{0.2ex}
\begin{algorithm}
{
\caption{Solve \mzcc{G}{\zipcoll{Z}}\label{algo:fpt}}
\renewcommand{\algorithmicrequire}{\textbf{Input:}}
\renewcommand{\algorithmicensure}{\textbf{Output:}}
\begin{algorithmic}[1]
\REQUIRE{Compatibility graph $G$, zipper constraints \zipcoll{Z}}
\ENSURE{Clique cover $\cover{K}_{\text{best}}$ with minimum cardinality}\\[6pt]

\STATE Take  $\D = \Zpairs \setminus \R$
\STATE Initialize $\cover{K}_{\text{best}}\leftarrow\big\{\{v\}\mid v\in V(G)\big\}$
\FOR {$\Selon{S}{\D} \in \enum(\prG{G}(\D))$}
\STATE Form $\eD$ and $\eSelon{S}{\eD}$ from $\Selon{S}{\D}$ (via Construction~\ref{const:bound-inclusion})\label{line:start-loop}
\STATE Build graph $G^{+}(\eSelon{S}{\eD})$ (via Construction~\ref{const:augmentation})
\STATE $\cover{K}^+ \leftarrow$ \fxn{Find-min-clique-cover}$(G^{+}(\eSelon{S}{\eD}))$ \label{line:std-min-clique}
\STATE Distill $\cover{K}^+$ to $\cover{K}$ (via Definition~\ref{def:distillation})
\STATE Repair $\cover{K}$ with $\uR = \R \setminus \eD$ to give $\cover{K}^\dagger$ (Lemma~\ref{lem:repairR})
\label{line:penultimate-loop}
\STATE $\cover{K}_{\text{best}} \leftarrow 
\cover{K}^\dagger$ when 
$|\cover{K}^\dagger| < |\cover{K}_{\text{best}}|$
\ENDFOR
\RETURN $\cover{K}_{\text{best}}$
\end{algorithmic}
}
\end{algorithm}
\vspace*{0.4ex}

Figure~\ref{fig:R_and_D} helps to illustrate the relationships
between the subsets of $\Zpairs$ appearing in the algorithm: 
$\D$ and $\R$ partition $\Zpairs$, and so too do 
$\eD$ and $\uR$.
Two additional points are worth noting. 
Though the domains are taken as $\eD$ when $G^+$ is constructed, 
and $\eD$  is usually larger than $\D$, it is only
the downstream enabled prescriptions on $\D$ that are
enumerated.
Secondly, line~\ref{line:std-min-clique} 
constructs a minimal clique cover on $G^+$, which is unconstrained. 
This sub-problem, though \nphard, is fixed-parameter tractable, and we will use the specific method mentioned in Lemma~\ref{existing:fpt} for the subsequent analysis appearing below.

\begin{lemma}[Constraint satisfaction]
\label{lemma:extension-of-zippercs-if-d-less-repairable}
Given $G$ and an associated $\Zpairs$ from zipper constraint collection $\zipcoll{Z}$, 
let $\R \subseteq \Zpairs$ be a set of pairs such that
every pair in $\R$ has \cn. 
With $\D \supseteq \Zpairs \setminus \R$, let 
$\Selon{S}{\D}$ be a
downstream-enabled prescription on domain $\D$.
If $\cover{K}^+$ is a 
cover of $G^{+}(\eSelon{S}{\eD})$
that is
faithful to
$\eSelon{S}{\eD}$,
then 
$\cover{K}^+$
can be distilled
and
repaired 
so that every
constraint in 
$\zipcoll{Z}$ is satisfied.
\end{lemma}
\begin{proof}
The repair process on the cover chooses $\uR = \R \setminus \eD$.
First, the collection of zipper constraints are of four types,
depending on whether the zipper constraints have both source and target in $\eD$ or not: 
(\emph{i}) zipper constraints with both source and target in
$\eD$,
(\emph{ii}) zipper constraints with
neither source nor target in $\eD$,
(\emph{iii}) zipper
constraints with only target in $\eD$, (\emph{iv}) zipper constraints with only source in $\eD$.
Next, we will show that all zipper constraints in the above types are
satisfied.  
The derived $\eSelon{S}{\eD}$ is downstream-enabled (according to Property~\ref{prop:downstream-preserved-under-boundary-inclusion}) and,
then, Lemma~\ref{lem:faith} applied on that prescription means that 
$\cover{K}$ (the distillation of $\cover{K}^{+}$) 
satisfies all zipper constraints internal to $\eD$.
Hence type (\emph{i}) are not violated.
The repair procedure ensures that 
type-(\emph{ii}) constraints hold (as type-(\emph{ii}) constraints fall entirely within $\uR$.) 
As per Lemma~\ref{lem:repairR}, 
the repair procedure may have the side-effect of altering some constraints so they no longer hold.
Those are constraints with source in that region
of $\R$ outside of $\eD$ and destinations in $\eD$, i.e.,
constraints of type (\emph{iii}). This is not a concern
as the source must be \on, and the destination must be \off.
The destination is downstream from the source.
Either the destination is itself in $\D$ or it is in $\eD\setminus \D$, but in this latter case, there must be a
downstream \off-pair in $\D$ which caused it to be created.
But then Construction~\ref{const:bound-inclusion} would
have placed the source into $\eD$, which contradicts the
criterion for being of type (\emph{iii}).
The argument for type (\emph{iv}) is symmetric: a violation
involves an upstream \on-pair either in $\D$ or $\eD \setminus \D$, where the latter case only arises from some
further upstream pair in $\D$ that is \on.
But then, as the destination of the type-(\emph{iv}) element is
a downstream of an \on pair, it must be \on (through  Construction~\ref{const:bound-inclusion}), contradicting
the criterion for being type (\emph{iv}).
Therefore, all zipper constraints in $\zipcoll{Z}$ are satisfied.
\end{proof}

\begin{figure}[t]
\centering
\includegraphics[width=1.00\linewidth, trim=435pt 10pt 10pt 0pt, clip]{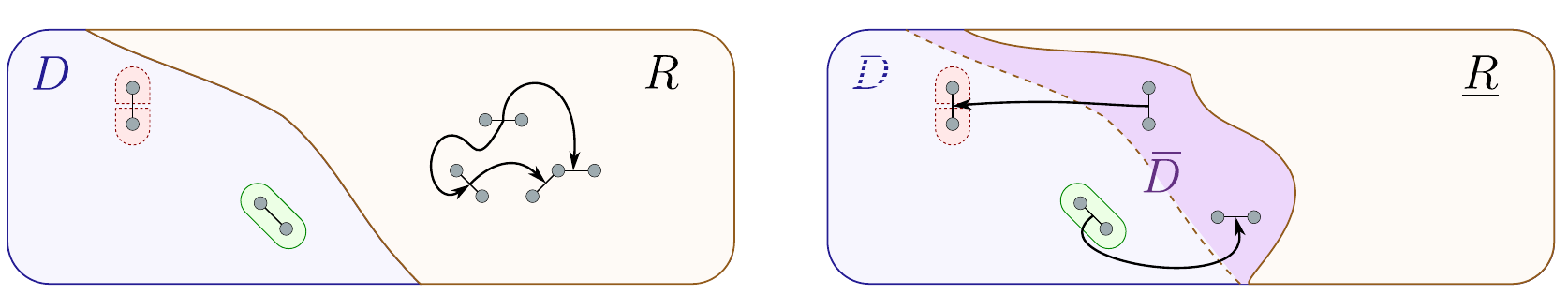}
\vspace*{-12pt}
\caption{Subsets of the set of $\Zpairs$ appearing in Algorithm~\ref{algo:fpt}.
First (on the left), $\Zpairs$ is partitioned into $\D$ and $\R$. 
Then (on the right), with some particular prescription given, $\D$ is grown to form $\eD$ by including those pairs which are upstream of \off elements and those pairs downstream of \on ones.
And $\uR$ is reduced by this same difference so  that
$\eD \cup \uR = \Zpairs$.\label{fig:R_and_D}}
\end{figure}

\begin{theorem}[Correctness]
\label{thrm:Dis}
Suppose $\cover{K}^\star$ is  a solution to \mzcc{G}{\zipcoll{Z}},
i.e., it is a minimum cardinality clique cover of
$G$ and satisfies $\zipcoll{Z}$. 
Then, there is a 
$\cover{K}^\dagger$, 
the repair of a distilled 
cover of 
a graph $G^{+}(\eSelon{S}{\eD})$,
constructed 
from $\eSelon{S}{\eD}$,
itself
being 
obtained as the boundary
inclusion of 
some 
downstream-enabled prescription 
from $\Selon{S}{\D}$ for
a domain $\D$ chosen
with $\D \supseteq \Zpairs \setminus \R$,
such that $|\cover{K}^\dagger| = |\cover{K}^\star|$.
\end{theorem}

\begin{proof}
Having chosen some domain $\D \supseteq \Zpairs \setminus \R$, we 
use $\cover{K}^{\star}$ to construct a downstream-enabled selection.
For any pair $\pr{v}{w} \in \D$, if there is a clique in $\cover{K}^{\star}$ that
contains both $v$ and $w$, then add
$\pr{v}{w}$ to $\Selon{S}{\D}$. (This `turns' them \on; otherwise, as the pair is in $\D$, 
they'll be \off.)
This process, 
having followed Definition~\ref{def:faithfull}, means that $\cover{K}^{\star}$ is faithful to $\Selon{S}{\D}$.
Next, perform the operations in lines \ref{line:start-loop}--\ref{line:penultimate-loop} of  Algorithm~\ref{algo:fpt}.
As distillation and repair operations do not increase the cover size, thus
$|\cover{K}^{+}|\geq |\cover{K}| \geq |\cover{K}^{\dagger}|$. 

Using Lemma~\ref{lem:opt-faith}, there is no clique cover on $G$ 
that is faithful to $\Selon{S}{\D}$ and is
smaller than $|\cover{K}|$. Therefore, we have $|\cover{K}^\star| \geq |\cover{K}|$.
But $\cover{K}^{\dagger}$ satisfies all the zipper constraints (via Lemma~\ref{lemma:extension-of-zippercs-if-d-less-repairable}), and with $\cover{K}^\star$ being the smallest
such cover,  $|\cover{K}^{\star}| \leq |\cover{K}^{\dagger}|$.
Combining: $|\cover{K}^{\dagger}| \leq |\cover{K}| \leq |\cover{K}^{\star}| \leq |\cover{K}^{\dagger}|$.
\end{proof}

\begin{corollary}
\label{cor:expr}
Algorithm~\ref{algo:fpt} is fixed-parameter tractable, having 
complexity:
\[ f(\beta, m,\ell,\omega, d)\, n^{O(1)}\]
with $f(\beta, m,\ell,\omega, d) = (2+\ell)^\omega 2^{(\beta+d) (m+d) \log (m+d)}$,
where $n$ is the size of the input graph $G$, $\beta$ is the size of the largest clique in $G$,
$m$ is the number of cliques in the minimum clique cover of $G$,
$\ell$ and $\omega$ are the height and width of 
\mbox{$(\Zpairs /\!\!\cyceq, \strictupstream)$}, respectively, and 
$d = |\D| + 
|\{P \in \Zpairs\setminus\D\;|\;\exists Q\in \D,  P\strictupstream Q \text{ or } Q \strictupstream P\}|$. 
\end{corollary}
\begin{proof}
Take $\D$ to be the entire $\Zpairs$. Following Algorithm~\ref{algo:enum}, there are $(2+\ell)^{\omega}$ downstream-enabled prescriptions. For each prescription, we can construct an augmented graph. Compared to the original graph $G$, the augmented graph creates at most $d$ additional states and keeps the copies of incompatible states incompatible. In the worst case, the copies of the vertices in the largest clique of $G$ still remains fully connected in the augmented graph. Hence, the size of the largest clique in the augmented graph is at most $m+d$, where there are additionally $d$ new states to be covered by the largest clique. For each state pair $\pr{v}{w}$ in $\Zpairs$, if it is \off, then it requires at most one additional clique since you can no longer put these two states in the same clique. If it is \on, then it also requires at most one additional clique to cover the additional new states $\newvertex{v}{w}$, as the copies of both $v$ and $w$ are already covered by the clique cover of size $\beta$. Therefore, the size of the minimum clique cover for the augmented graph is at most $\beta+d$. Hence, the computational complexity to find the minimum clique cover for each augmented graph is $2^{(\beta+d) (m+d) \log (m+d)}O(n)$ following Lemma~\ref{existing:fpt}. Together, the complexity for Algorithm~\ref{algo:fpt} is $(2+\ell)^{\omega}2^{(\beta+d) (m+d) \log (m+d)}O(n)$.
\end{proof}

Notice that the approach does not depend upon any particular details of
the \fpt-algorithm employed to find the traditional clique cover. 
For  Corollary~\ref{cor:expr}'s
precise expression of $f$,  we use Lemma~\ref{existing:fpt} as a specific example, 
and modifications for the $G^+$ graphs add only $d$ terms to upper bound the parameters.
In a sense, we can see the compositionality of
the \fpt theory:  in order to 
account for the enumeration,
the zipper constraints themselves contribute to the complexity via the $(2+\ell)^\omega = 2^{\omega \lg(2 + \ell)}$ factor. 

\section{Conclusion and outlook}

It is unclear whether the algorithm we have presented
is of particular practical value. Given past successes
with ILP- and SAT-based formulations, and the vast body of active work on improving solvers of
those sorts, they may well outperform direct treatment via clique covers on graphs.  
Nevertheless, what the present algorithm \emph{does} provide is some deeper understanding of the fact that the constraints to enforce determinism play a role in making
the problem hard.
To gain further insight, one might look at regularity which 
affects the down-/upstream relationship between
zipper constraints, and examine its impact on
the chains and anti-chains that result.
Under the usual interpretation of \fpt,
such regularity 
leads one to identify classes of tractable instances with complexity characterized by constant parameters.  
These instances are efficient to solve
when scaling the problem while keeping those parameters fixed.
Finally, the notion of repairability  in  \cite{zhang23general} has definitely been sharpened within the present paper, though, unlike that work,
our emphasis has not been on  the direct structural aspects of the
compatibility graph.

\vspace*{12pt}
\noindent\textbf{Acknowledgements:}
We thank the anonymous referees for their close reading of the manuscript,
and acknowledge the support of the Office of Naval Research under Award \#N00014-22-1-2476.

\renewcommand{\UrlFont}{\ttfamily\scriptsize}
\bibliographystyle{plain}
\bibliography{mybib}

\end{document}